\documentclass[11pt]{article} 
\usepackage[margin=1in]{geometry}
\usepackage{amsmath,amssymb,amsthm}
\usepackage[usenames]{color}
\usepackage{clrscode}
\usepackage{epsfig}

\newtheorem{theorem}{Theorem}

\newtheorem{claim}{Claim}

\newtheorem{lemma}{Lemma}

\newtheorem{remark}{Remark}
\newtheorem*{theoremK}{Khinchine's Inequality}

\def\ie{{\it i.e.~}}
\def\th{^{\mbox{\footnotesize{\it th}}}}
\def\ith{i^{\mbox{\footnotesize{\it th}}}}

\def\vecpt{\mvec{p}^t}
\def\vecp{\mvec{p}}
\def\vecx{\mvec{x}}
\def\veceta{\mvec{\eta}}
\def\E{\mathbb{E}}

\def\Bern{\mathrm{Bernoulli}}

\def\FPL{\mathsf{FPL}}
\def\DFPL{\mathsf{DFPL}}
\def\HYZ{\mathsf{HYZ}}
\def\FR{\mathrm{FR}}
\def\FP{\mathrm{FP}}
\def\reals{\mathbb{R}}
\def\ftpl{\mathsf{FPL}}

\def\LEF{\mathsf{LEF}}

\newcommand{\bsection}[1]{\vskip -2pt \section{#1}}

\newcommand{\vecP}{\mathbf{P}}
\newcommand{\vecQ}{\mathbf{Q}}

\newcommand{\alg}{\mathbf{A}}
\newcommand{\payoff}{\mathrm{Payoff}}

\newcommand{\remove}[1] {}

\newcommand{\br}[1]{{\color{blue} {\bf BR Edit}: \it #1}}

\newcommand{\mvec}[1]{{\bf #1}}
\newcommand{\eat}[1]{ }

\newcommand{\rfull}{R_{\mathrm{Full}}}

\title{Distributed Non-Stochastic Experts}
\author{
Varun Kanade\thanks{This research was carried out while the author was at
Harvard University supported in part by grant NSF-CCF-09-64401 }\\
UC Berkeley\\
\texttt{vkanade@eecs.berkeley.edu} \\
\and
Zhenming Liu\thanks{This research was carried out while the author was at
Harvard University supported in part by grants NSF-IIS-0964473 and
NSF-CCF-0915922} \\
Princeton University\\
\texttt{zhenming@cs.princeton.edu} \\
\and
Bo\v{z}idar Radunovi\'c \\
Microsoft Research\\
\texttt{bozidar@microsoft.com} \\
}

\begin{document}

\maketitle

\begin{abstract}
We consider the online distributed non-stochastic experts problem, where the
distributed system consists of one \emph{coordinator} node that is connected to
$k$ \emph{sites}, and the sites are required to communicate with each other via
the coordinator. At each time-step $t$, one of the $k$ site nodes has to pick
an expert from the set $\{1, \ldots, n\}$, and the same site receives
information about payoffs of all experts for that round. The goal of the
distributed system is to minimize regret at time horizon $T$, while
simultaneously keeping communication to a minimum.
The two extreme solutions to this problem are: (i) \emph{Full communication}:
This essentially simulates the non-distributed setting to obtain the optimal
$O(\sqrt{\log(n)T})$ regret bound at the cost of $T$ communication.
%
(ii) \emph{No communication}: Each site runs an independent copy -- the
regret is $O(\sqrt{\log(n)kT})$ and the communication is $0$.
%
This paper shows the difficulty of simultaneously achieving regret
asymptotically better than $\sqrt{kT}$ and communication better than $T$. We
give a novel algorithm that for an oblivious adversary achieves a non-trivial
trade-off: regret $O(\sqrt{k^{5(1 + \epsilon)/6}T})$ and communication
$O(T/k^{\epsilon})$, for any value of $\epsilon \in (0, 1/5)$.
We also consider a variant of the model, where the coordinator picks the
expert.  In this model, we show that the \emph{label-efficient} forecaster of Cesa-Bianchi et al. (2005)
already gives us strategy that is near optimal in regret vs communication trade-off.

\end{abstract}

\bsection{Introduction}
In this paper, we consider the well-studied non-stochastic expert problem in a
distributed setting. In the standard (non-distributed) setting, there are a
total of $n$ experts available for the decision-maker to consult, and at each
round $t = 1, \ldots, T$, she must choose to follow the \emph{advice} of one of
the experts, say $a^t$, from the set $[n] = \{1, \ldots, n\}$.  At the end of
the round, she observes a payoff vector $\vecp^t \in [0, 1]^n$, where
$\vecp^t[a]$ denotes the payoff that would have been received by following the
advice of expert $a$. The payoff received by the decision-maker is
$\vecp^t[a^t]$. In the \emph{non-stochastic} setting, an adversary decides
the payoff vectors at any time step. At the end of the $T$ rounds, the
\emph{regret} of the decision maker is the difference in the payoff that she
would have received using the single best expert at all times in hindsight, and
the payoff that she actually received, i.e. $R = \max_{a \in [n]} \sum_{t=1}^T
\vecpt[a] - \sum_{t=1}^T \vecpt[a^t]$. The goal here is to
minimize her regret; this general problem in the non-stochastic setting captures
several applications of interest, such as experiment design, online
ad-selection, portfolio optimization, etc. (See \cite{FS95, CBL06, Cov91, HK09,
Hazan12} and references therein.)

Tight bounds on regret for the non-stochastic expert problem are obtained by the
so-called follow the \emph{regularized} leader approaches; at time $t$, the
decision-maker chooses a distribution, $\vecx^t$, over the $n$ experts.  Here
$\vecx^t$ minimizes the quantity $\sum_{s=1}^{t-1} \vecp^t \cdot \vecx +
r(\vecx)$, where $r$ is a regularizer. Common regularizers are the entropy
function, which results in Hedge~\cite{FS95} or the exponentially weighted
forecaster~(see chap. 2 in \cite{CBL06}), or as we consider in this paper $r(x)
= \bar{\veceta} \cdot \vecx$, where $\bar{\veceta} \in_R [0, \eta]^n$ is a
random vector, which gives the follow the perturbed leader ($\ftpl$)
algorithm~\cite{KV05}.

We consider the setting when the decision maker is a distributed system, where
several different nodes may select experts and/or observe payoffs at different
time-steps. Such settings are common, e.g. internet search companies, such as
Google or Bing, may use several nodes to answer search queries and the
performance is revealed by user clicks. From the point of view of making better
predictions, it is useful to pool all available data. However, this may involve
significant communication which may be quite costly. Thus, there is an obvious
trade-off between cost of communication and cost of inaccuracy (because of not
pooling together all data), which leads to the question:
{\quote{\em
What is the \emph{explicit} trade-off between
the total amount of communication needed and the regret of the expert problem under
\emph{worst case} input?
}}
\section{Models and Summary of Results}
We consider a distributed computation model consisting of one central
\emph{coordinator} node connected to $k$ site nodes. The site nodes must
communicate with each other using the coordinator node. At each time step, the
distributed system receives a \emph{query}\footnote{We do not use the word query
in the sense of explicitly giving some information or context, but merely as
indication of occurrence of an event that forces some site or coordinator to
choose an expert. In particular, if any context is provided in the query the
algorithms considered in this paper ignore all context -- thus we are in the
\emph{non-contextual} expert setting.}, which indicates that it must choose an
expert to follow. At the end of the round, the distributed system observes the
payoff vector. We consider two different models described in detail below: the
\emph{site prediction model} where one of the $k$ sites receives a query at any
given time-step, and the \emph{coordinator prediction} model where the query is
always received at the coordinator node. In both these models, the payoff
vector, $\vecp^t$, is always observed at one of the $k$ site nodes.  Thus, some
communication is required to share the information about the payoff vectors
among nodes.  As we shall see, these two models yield different algorithms and
performance bounds.  \medskip \\

\noindent{\bf Goal}: The algorithm implemented on the distributed system may use
randomness, both to decide which expert to pick and to decide when to
communicate with other nodes.  We focus on simultaneously minimizing the
expected regret and the expected communication used by the (distributed)
algorithm.  Recall, that the expected regret is:
{\small
\begin{align}
\E[R] &= \E\left[\max_{a \in [n]} \sum_{t=1}^T  \vecpt[a] - \sum_{t=1}^T
\vecpt[a^t],
\right]
\end{align}
}
where the expectation is over the random choices made by the algorithm. The
expected communication is simply the expected number (over the random choices)
of messages sent in the system.

As we show in this paper, this is a challenging problem and to keep the analysis
simple we focus on bounds in terms of the number of sites $k$ and the time
horizon $T$, which are often the most important scaling parameters.  In
particular, our algorithms are variants of \emph{follow the perturbed leader}
($\ftpl$) and hence our bounds are not optimal in terms of the number of experts
$n$. We believe that the dependence on the number of experts in our algorithms
(upper bounds) can be strengthened using a different regularizer.  Also, all our
lower bounds are shown in terms of $T$ and $k$, for $n = 2$. For larger $n$,
using techniques similar to Theorem 3.6 in \cite{CBL06} should give the
appropriate dependence on $n$. \medskip \\

\noindent{\bf Adversaries}: In the non-stochastic setting, we assume that an
adversary may decide the payoff vectors, $\vecp^t$, at each time-step and also
the site, $s^t$, that receives the payoff vector (and also the query in the
site-prediction model).
An \emph{oblivious adversary} cannot see any of the actions of the distributed
system, i.e. selection of expert, communication patterns or any random bits
used. However, the oblivious adversary may know the description of the
algorithm.  In addition to knowing the description of the algorithm, an
\emph{adaptive adversary} is stronger and can record all of the past actions of
the algorithm, and use these arbitrarily to decide the future payoff vectors and
site allocations. \smallskip \\

\noindent{\bf Communication}: We do not explicitly account for message sizes.
However, since we are interested in scaling with $T$ and $k$, we do require that
message size should not depend on the number of sites $k$ or the number of
time-steps $T$, but only on the number of experts $n$.  In other words, we
assume that $n$ is substantially smaller than $T$ and $k$.  All the messages
used in our algorithms contain at most $n$ real numbers. As is standard in the
distributed systems literature, we assume that communication delay is $0$, i.e.
the updates sent by any node are received by the recipients before any future
query arrives. All our results still hold under the weaker assumption that the
number of queries received by the distributed system in the duration required to
complete a broadcast is negligible compared to $k$.~\footnote{This is because in
regularized leader like approaches, if the cumulative payoff vector changes by a
small amount the distribution over experts does not change much because of the
\emph{regularization} effect.}

We now describe the two models in greater detail, state our main results and
discuss related work: \medskip \\

\noindent{\sc {1. \textbf{Site Prediction Model}}}: At each time step $t = 1,
\ldots, T$, one of the $k$ sites, say $s^t$, receives a \emph{query} and has to
pick an expert, $a^t$, from the set, $[n] = \{1, \ldots, n\}$. The payoff vector
$\vecp^t \in [0, 1]^n$, where $\vecp^t[i]$ is the payoff of the $\ith$ expert is
revealed \emph{only to the site} $s^t$ and the decision-maker (distributed
system) receives payoff $\vecp^t[a^t]$, corresponding to the expert actually
chosen. The site prediction model is commonly studied in distributed machine
learning settings (see \cite{DGBSX11,DAW10,AD11}).  The payoff vectors, $\vecp^1, \ldots,
\vecp^T$, and also the choice of sites that receive the query, $s^1, \ldots,
s^T$, are decided by an adversary. There are two very simple
algorithms in this model: \medskip \\
(i) \emph{Full communication}: The coordinator always maintains the current
cumulative payoff vector, $\sum_{\tau=1}^{t-1} \vecp^\tau$. At time step $t$,
$s^t$ receives the current cumulative payoff vector $\sum_{\tau=1}^{t-1}
\vecp^{\tau}$ from the coordinator, chooses an expert $a^t \in [n]$ using
$\ftpl$, receives payoff vector $\vecp^t$ and sends $\vecpt$ to the
coordinator, which updates its cumulative payoff vector. Note that the total
communication is $2T$ and the system simulates (non-distributed)
$\ftpl$ to achieve (optimal) regret guarantee $O(\sqrt{nT})$. \smallskip \\
(ii) \emph{No communication}: Each site maintains cumulative payoff vectors
corresponding to the queries received by them, thus implementing $k$ independent
versions of $\ftpl$. Suppose that the $\ith$ site receives a total of $T_i$
queries ($\sum_{i=1}^k T_i = T$), the regret is bounded by $\sum_{i=1}^k
O(\sqrt{n T_i}) = O(\sqrt{nkT})$ and the total communication is $0$. This upper
bound is actually tight, as shown in Lemma \ref{lemma:sp-no-comm} (Appendix
\ref{app:sp-lb-zero}), in the event that there is $0$ communication.

Simultaneously achieving regret that is asymptotically lower than $\sqrt{knT}$
using communication asymptotically lower than $T$ turns out to be a
significantly challenging question.
Our main positive result is the first distributed expert algorithm in the
\emph{oblivious adversarial} (non-stochastic) setting, using \emph{sub-linear}
communication. Finding such an algorithm in the case of an adaptive adversary is
an interesting open problem.
\begin{theorem}\label{thm:main} When $T \geq 2 k^{2.3}$, there exists an
algorithm for the distributed experts problem that against an oblivious
adversary achieves regret $O(\log(n) \sqrt{k^{5(1+\epsilon)/6} T})$ and uses
communication $O(T/k^{\epsilon})$, giving non-trivial guarantees in the range
$\epsilon \in (0, 1/5)$.  \end{theorem} \medskip
%
%

\noindent{\sc 2. \textbf{Coordinator Prediction Model}}: At every time step,
the query is received by the coordinator node, which chooses an expert $a^t \in
[n]$. However, at the end of the round, one of the site nodes, say $s^t$,
observes the payoff vector $\vecp^t$. The payoff vectors $\vecp^t$ and choice of
sites $s^t$ are decided by an adversary. This model is also a natural one and is
explored in the distributed systems and streaming literature (see \cite{CMY11,
HYZ12, LRV12} and references therein).

The \emph{full communication} protocol is equally applicable here getting
optimal regret bound, $O(\sqrt{nT})$ at the cost of substantial (essentially
$T$) communication. But here, we do not have any straightforward algorithms that
achieve non-trivial regret without using any communication.  This model is
closely related to the label-efficient prediction problem (see Chapter 6.1-3 in
\cite{CBL06}), where the decision-maker has a limited budget and has to spend
part of its budget to observe any payoff information. The optimal strategy is to
request payoff information randomly with probability $C/T$ at each time-step, if
$C$ is the communication budget. We refer to this algorithm as $\LEF$
(label-efficient forecaster) \cite{CBLS05}.
\begin{theorem} \cite{CBLS05} (Informal) The $\LEF$ algorithms using $\ftpl$ with communication budget $C$
achieves regret $O(T \sqrt{n/C})$ against both an adaptive and an oblivious
adversary.  \end{theorem}
One of the crucial differences between this model and that of the
label-efficient setting is that when communication does occur, the site can send
cumulative payoff vectors comprising all previous updates to the coordinator
rather than just the latest one.  The other difference is that, unlike in the
label-efficient case, the sites have the knowledge of their local regrets and
can use it to decide when to communicate. However, our lower bounds for natural
types of algorithms show that these advantages probably do not help to get
better guarantees. \medskip \\

%
\noindent{\bf Lower Bound Results}:
In the case of an \emph{adaptive adversary}, we have an unconditional (for any
type of algorithm) lower bound in both the models:
\begin{theorem} \label{thm:adaptive} Let $n = 2$ be the number of experts. Then
any (distributed) algorithm that achieves expected regret $o(\sqrt{kT})$ must
use communication $(T/k)(1 - o(1))$.
\end{theorem}
\vskip -6pt
The proof appears in Appendix \ref{app:adaptive}. Notice that in the coordinator prediction model,
when $C = T/k$, this lower bound is matched by the upper bound of $\LEF$.
%

In the case of an oblivious adversary, our results are weaker, but we can show
that certain natural types of algorithms are not applicable directly in this
setting. The so called \emph{regularized} leader algorithms, maintain a
cumulative payoff vector, $\vecP^t$, and use only this and a regularizer to
select an expert at time $t$. We consider two variants in the distributed
setting: \smallskip \\
(i) {\em Distributed Counter Algorithms}: Here the forecaster only
 uses $\tilde{\vecP}^t$, which is an (approximate) version of the
cumulative payoff vector $\vecP^t$. But we make no assumptions on how
the forecaster will use $\tilde{\vecP}^t$. $\tilde{\vecP}^t$ can be maintained while
using sub-linear communication by applying techniques from distributed systems
literature~\cite{HYZ12}. \smallskip \\
(ii) {\em Delayed Regularized Leader}: Here the regularized leaders don't try to
explicitly maintain an approximate version of the cumulative payoff vector.
Instead, they may use an arbitrary communication protocol, but make prediction
using the cumulative payoff vector (using \emph{any} past payoff vectors that
they could have received) and some regularizer.

We show in Section~\ref{sec:lower_bound} that the distributed counter approach
does not yield any non-trivial guarantee in the site-prediction model even
against an \emph{oblivious} adversary. It is possible to show a similar lower
bound the in the coordinator prediction model, but is omitted since it follows
easily from the idea in the site-prediction model combined with an explicit
communication lower bound given in \cite{HYZ12}.

Section~\ref{sec:cp} shows that the delayed regularized leader approach
does not yield non-trivial guarantees even against an \emph{oblivious
adversary} in the coordinator prediction model, suggesting $\LEF$ algorithm is
near optimal. \medskip \\

\noindent {\bf Related Work}: Recently there has been significant interest in
distributed online learning questions (see for example \cite{DGBSX11,DAW10,AD11}).
However, these works have focused mainly on stochastic optimization problems.
Thus, the techniques used, such as reducing variance through mini-batching, are
not applicable to our setting. Questions such as network structure~\cite{DAW10}
and network delays~\cite{AD11} are interesting in our setting as well, however,
at present our work focuses on establishing some non-trivial regret guarantees
in the distributed online non-stochastic experts setting.
Study of communication as a resource in distributed learning is also considered in
\cite{BBFM12,DPSV12a, DPSV12b}; however, this body of work seems only applicable
to offline learning.

The other related work is that of distributed functional monitoring \cite{CMY11}
and in particular distributed counting\cite{HYZ12,LRV12}, and sketching
~\cite{Cor11}. Some of these techniques have been successfully applied in
offline machine learning problems~\cite{CHW10}.  However, we are the first to
analyze the performance-communication trade-off of an online learning algorithm
in the standard distributed functional monitoring framework~\cite{CMY11}.  An
application of a distributed counter to an online Bayesian regression was
proposed in Liu et al.~\cite{LRV12}.  Our lower bounds discussed below, show
that approximate distributed counter techniques do not directly yield
non-trivial algorithms.

%


\section{Site-prediction model}
\label{sec:sp}
\subsection{Upper Bounds}
\label{sec:upper_bound}
We describe our algorithm that simultaneously achieves non-trivial bounds on
expected regret and expected communication.  We begin by making two assumptions
that simplify the exposition.  First, we assume that there are only $2$ experts.
The generalization from $2$ experts to $n$ is easy, as discussed in the
Remark~\ref{rem:n-experts} at the end of this section.  Second, we assume that
there exists a global query counter, that is available to all sites and the
co-ordinator, which keeps track of the total number of queries received across
the $k$ sites. We discuss this assumption in Remark~\ref{rem:sp-counter} at the
end of the section.  As is often the case in online algorithms, we assume that
the time horizon $T$ is known. Otherwise, the standard doubling trick may be
employed. The notation used in this Section is defined in
Table~\ref{table:notation}. 

\noindent{\bf Algorithm Description}: 
Our algorithm $\DFPL$ is described in Figure~\ref{alg:DFPL}(a). We
make use of $\ftpl$ algorithm, described in Figure~\ref{alg:DFPL}(b),
which takes as a parameter the amount of added noise $\eta$.  $\DFPL$
algorithm treats the $T$ time steps as $b (= T/\ell)$~blocks, each of
length $\ell$.  At a high level, with probability $q$ on any given
block the algorithm is in the {\em step} phase, running a copy of
$\ftpl$ (with noise parameter $\eta^\prime$) across all time steps of
the block, synchronizing after each time step.  Otherwise it is in
a {\em block} phase, running a copy of $\ftpl$ (with noise parameter
$\eta$) across blocks with the same expert being followed for the
entire block and synchronizing after each block.  This effectively
makes $\vecP^i$, the cumulative payoff over block $i$, the payoff
vector for the block $\FPL$. The block $\FPL$ has on average $(1-q)
T/\ell$ total time steps.  We begin by stating a (slightly stronger)
guarantee for $\ftpl$.

{\footnotesize
\begin{table}
\begin{center}
\begin{tabular}{cl}\hline
{\bf Symbol} & {\bf Definition} \\  \hline
$\vecp^t$ & Payoff vector at time-step $t$, $\vecp^t \in [0, 1]^2$\\
$\ell$    & The length of block into which inputs are divided \\
$b$	& Number of input blocks $b = T/\ell$ \\
$\vecP^i$ & Cumulative payoff vector within block $i$, $\vecP^i =
\sum_{t=(i-1)\ell + 1}^{i\ell} \vecp^t$ \\
$\vecQ^i$ & Cumulative payoff vector until end of block $(i-1)$, $\vecQ^i =
\sum_{j=1}^{i-1} \vecP^j$ \\
$M(v)$ & For vector $v \in \reals^2$, $M(v) = 1$ if $v_1 > v_2$; $M(v) = 2$
otherwise \\
$\FP^i(\eta)$ & Random variable denoting the payoff obtained by playing
$\FPL(\eta)$ on block $i$ \\
$\FR^i_a(\eta)$ & Random variable denoting the regret with respect to action $a$
of playing $\FPL(\eta)$ on block $i$ \\
 & $\FR^i_a(\eta) = \vecP^i[a] - \FP^i(\eta)$ \\ 
$\FR^i(\eta)$ & Random variable denoting the regret of playing $\FPL(\eta)$ on
payoff vectors in block $i$ \\
 & $\FR^i(\eta) = \max_{a =1, 2} \vecP^i[a] - \FP^i(\eta) = \max_{a =1, 2} \FR^i_a(\eta)$ \\ \hline
\end{tabular}
\caption{Notation used in Algorithm $\DFPL$ (Fig. \ref{alg:DFPL}) and in
Section \ref{sec:upper_bound}. \label{table:notation}}
\end{center}
\end{table}
}

%
\begin{figure}
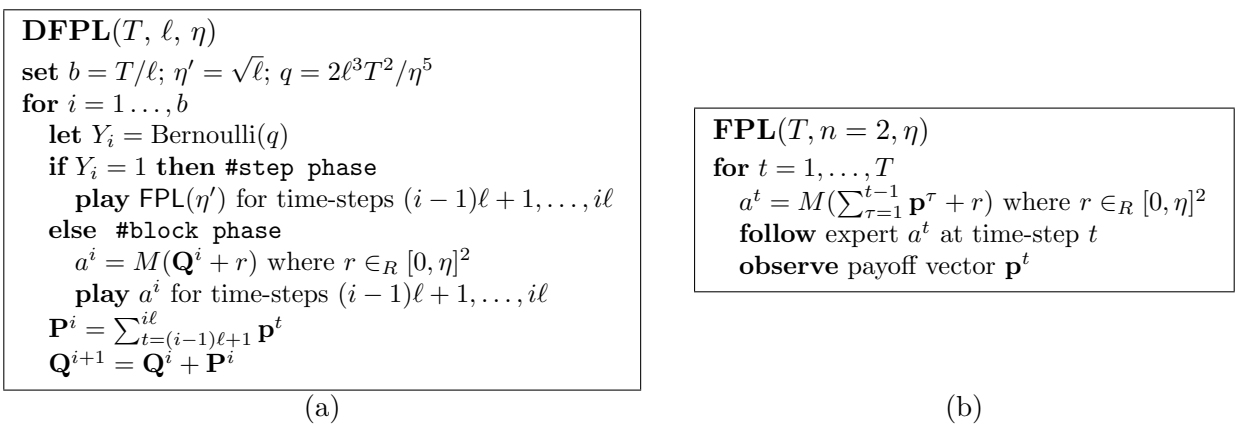

\centering{
$
\begin{array}{ccc}
\fbox{
{\small
\begin{minipage}{0.4 \textwidth}
\begin{tabbing}
aa\=aa\=aa\=aa\=aa\= \kill
{\normalsize {\bf DFPL}($T$, $\ell$, $\eta$)} \\[2pt]
{\bf set} $b = T/\ell$; $\eta^\prime = \sqrt{\ell}$; $q = 2 \ell^3 T^2/ \eta^5$ \\
{\bf for} $i = 1 \ldots, b$ \+ \\
{\bf let} $Y_i = \Bern(q)$ \\
{\bf if} $Y_i = 1$ {\bf then} {\tt\#step phase} \+ \\
{\bf play} $\FPL(\eta^\prime)$ for time-steps $(i-1)\ell + 1, \ldots, i\ell$ \- \\
{\bf else } {\tt\#block phase} \+ \\
$a^i = M(\vecQ^{i} + r)$ where $r \in_R [0, \eta]^2$ \\
{\bf play} $a^i$ for time-steps $(i-1)\ell + 1, \ldots, i\ell$ \- \\
$\vecP^{i} = \sum_{t=(i-1)\ell + 1}^{i\ell} \vecp^t$ \\
$\vecQ^{i+1} = \vecQ^{i} + \vecP^i$
\end{tabbing}
\end{minipage}
}
} & &
\fbox{
{\small
\begin{minipage}{0.4 \textwidth}
\begin{tabbing}
aa\=aa\=aa\=aa\=aa\= \kill
{\normalsize {\bf FPL}$(T,n=2,\eta)$} \\[2pt]
{\bf for} $t = 1, \ldots, T$ \+ \\
$a^t = M(\sum_{\tau=1}^{t-1} \vecp^{\tau} + r)$ where $r \in_R [0, \eta]^2$ \\ 
{\bf follow} expert $a^t$ at time-step $t$  \\
{\bf observe} payoff vector $\vecp^t$ \-
\end{tabbing}
\end{minipage}
}
}  \\
\mbox{(a)} & & \mbox{(b)}
\end{array}$
}

\caption{\label{alg:DFPL}{\small (a) $\DFPL$: Distributed Follow the Perturbed
Leader, (b) $\FPL$: Follow the Perturbed Leader with parameter $\eta$ for $2$
experts ($M(\cdot)$ is defined in Table \ref{table:notation}, $r$ is a random
vector)}}
\end{figure}

\begin{lemma} \label{lemma:ftpl-refined} Consider the case $n = 2$. Let $\vecp^1, \ldots, \vecp^T \in
[0,1]^2$ be a sequence of payoff vectors such that $\max_t |\vecp^t|_\infty \leq
B$ and let the number of experts be $2$. Then $\FPL(\eta)$ has the following
guarantee on expected regret, $\E[R] \leq \frac{B}{\eta} \sum_{t=1}^T
|\vecp^t[1] - \vecp^t[2]| + \eta$. \end{lemma}
The proof is a simple modification to the proof of the standard
analysis~{\cite{KV05}} and is given in Appendix \ref{app:ftpl-proof} for
completeness.  The rest of this section is devoted to the proof of Lemma~\ref{lem:DFPL}

\begin{lemma} \label{lem:DFPL} Consider the case $n=2$. If $T > 2 k^{2.3}$, Algorithm $\DFPL$ (Fig.
\ref{alg:DFPL}) when run with parameters $\ell$, $T$, $\eta = \ell^{5/12}T^{1/2}$
and $b, \eta^\prime, q$ as defined in Fig~\ref{alg:DFPL}, has expected regret
$O(\sqrt{\ell^{5/6}T})$ and expected communication $O(Tk/\ell)$. In particular
for $\ell = k^{1 + \epsilon}$ for $0 < \epsilon < 1/5$, the algorithm
simultaneously achieves regret that is asymptotically lower than $\sqrt{kT}$ and
communication that is asymptotically lower\footnote{Note that here asymptotics
is in terms of both parameters $k$ and $T$. Getting communication of the form
$T^{1-\delta} f(k)$ for regret bound better than $\sqrt{kT}$, seems to be a
fairly difficult and interesting problem} than $T$.
\end{lemma}

Since we are in the case of an
oblivious adversary, we may assume that the payoff vectors $\vecp^1, \ldots,
\vecp^T$ are fixed ahead of time. Without loss of generality let expert $1$ (out
of $\{1, 2\}$) be the one that has greater payoff in hindsight. 
Recall that $\FR^i_1(\eta^\prime)$ denotes the random variable that is the
regret of playing $\ftpl(\eta^\prime)$ {in a step phase} on block $i$ with
respect to the \emph{first} expert. In particular, this will be negative if
expert $2$ is the best expert on block $i$, even though globally expert $1$ is
better.  In fact, this is exactly what our algorithm exploits: {it gains on
regret in the communication-expensive, step phase while saving on communication
in the block phase.}

The regret can be written as 
\[ R = \sum_{i=1}^b \left( Y_i \cdot \FR^i_1(\eta^\prime) + (1 - Y_i)
(\vecP^i[1] - \vecP^i[a^i] \right). \]
Note that the random variables $Y_i$ are independent of the random variables
$\FR^i_1(\eta^\prime)$ and the random variables $a^i$. As $\E[Y_i] = q$, we can
bound the expression for expected regret as follows:
%
\begin{align}
\E[R] &\leq q \sum_{i=1}^b \E[\FR^i_1(\eta^\prime)] + (1 - q) \sum_{i=1}^b \E[\vecP^i[1] -
\vecP^i[a^i]] \label{eq:regret-main0}
\end{align}
We first analyze the second term of the above equation. This is just the regret
corresponding to running $\FPL(\eta)$ at the block level, with $T/\ell$ time
steps. Using the fact that $\max_i |\vecP^i|_{\infty} \leq \ell \max_t
|\vecp^t|_\infty \leq \ell$, Lemma \ref{lemma:ftpl-refined} allows us to
conclude that:
%
\begin{align}
\sum_{i=1}^b \E[\vecP^i[1] - \vecP^i[a^i]] &\leq \frac{\ell}{\eta}
\sum_{i=1}^b |\vecP^i[1] - \vecP^i[2]| + \eta \label{eq:regret-main2} 
\end{align}

Next, we also analyse the first term of the inequality (\ref{eq:regret-main0}).
We chose $\eta^\prime = \sqrt{\ell}$ (see Fig.~\ref{alg:DFPL}) and the analysis
of $\FPL$ guarantees that $\E[\FR^i(\eta^\prime)] \leq 2\sqrt{\ell}$, where
$\FR^i(\eta^\prime)$ denotes the random variable that is the \emph{actual}
regret of $\FPL(\eta^\prime)$, not the regret with respect to expert $1$ (which
is $\FR^i_1(\eta^\prime)$). Now either $\FR^i(\eta^\prime) =
\FR^i_1(\eta^\prime)$ (i.e. expert $1$ was the better one on block $i$), in
which case $\E[\FR^i_1(\eta^\prime)] \leq 2\sqrt{\ell}$; otherwise
$\FR^i(\eta^\prime) = \FR^i_2(\eta^\prime)$ (i.e.  expert $2$ was the better one
on block $i$), in which case $\E[\FR^i_1(\eta^\prime)] \leq  2\sqrt{\ell} +
\vecP^i[1] - \vecP^i[2]$. Note that in this expression $\vecP^i[1] - \vecP^i[2]$
is negative. Putting everything together we can write that
$\E[\FR^i_1(\eta^\prime)] \leq 2 \sqrt{\ell} - (\vecP^i[2] - \vecP^i[1])_+$,
where $(x)_+ = x$ if $x \geq 0$ and $0$ otherwise. Thus, we get the main
equation for regret.
\begin{align}
\E[R] &\leq 2 q b \sqrt{\ell} - \underbrace{q \sum_{i=1}^b (\vecP^i[2] -
\vecP^i[1])_+}_\text{term 1} +
\underbrace{\frac{\ell}{\eta} \sum_{i=1}^b |\vecP^i[1] -
\vecP^i[2]|}_\text{term 2} + \eta
\label{eq:regret-main}
\end{align}


Note that the first (i.e. $2qb\sqrt{\ell}$) and last (i.e. $\eta$) terms of
inequality (\ref{eq:regret-main}) are $O(\sqrt{\ell^{5/6}T})$ for the setting of
the parameters as in Lemma~\ref{lem:DFPL}. The strategy is to show that when
``term 2'' becomes large, then ``term 1'' is also large in magnitude, but
negative, compensating the effect of ``term 1''. We consider a few cases:
\smallskip \\
\noindent{\bf Case 1}: {\it When the best expert is identified quickly and not
changed thereafter.} 
%
Let $\zeta$ denote the maximum index, $i$, such that $\vecQ^i[1] - \vecQ^i[2]
\leq \eta$. Note that after the block $\zeta$ is processed, the algorithm in the
\emph{block} phase will never follow expert $2$.

Suppose that $\zeta \leq (\eta/\ell)^2$.  We note that the correct bound for
``term 2'' is now actually $(\ell/{\eta}) \sum_{i=1}^\zeta |\vecP^i[1] -
\vecP^i[2]|  \leq ({\ell^2 \zeta}/{\eta}) \leq \eta$
since $|\vecP^i[1] - \vecP^i[2]| \leq \ell$ for all $i$.
\smallskip \\
\noindent {\bf Case 2} {\it The best expert may not be identified quickly,
furthermore $|\vecP^i[1] - \vecP^i[2]|$ is large often.}
%
In this case, although ``term 2'' may be large (when $(\vecP^i[1]  -
\vecP^i[2])$ is large), this is compensated by the negative regret in ``term 1''
in expression (\ref{eq:regret-main}). This is because if $|\vecP^i[1] -
\vecP^i[2]|$ is large often, but the best expert is not identified quickly,
there must be enough blocks on which $(\vecP^i[2] - \vecP^i[1])$ is positive and
large.
%

Notice that $\zeta \geq (\eta/\ell)^2$.  Define $\lambda = \eta^2/T$ and let $S
= \{i \leq \zeta ~|~ |\vecP^i[1] - \vecP^i[2]| \geq \lambda \}$. Let $\alpha =
|S|/\zeta$. We show that $\sum_{i=1}^\zeta (\vecP^i[2] - \vecP^i[1])_+ \geq
(\alpha \zeta \lambda)/2 - \eta$.
To see this consider $S_1 = \{ i \in S ~|~ \vecP^i[1] > \vecP^i[2] \}$ and $S_2
= S \setminus S_1$. First, observe that $\sum_{i \in S} |\vecP^i[1] -
\vecP^i[2]| \geq \alpha \zeta \lambda$. Then, if $\sum_{i \in S_2} (\vecP^i[2] -
\vecP^i[1]) \geq (\alpha\zeta\lambda)/2$, we are done. If not $\sum_{i \in S_1}
(\vecP^i[1] - \vecP^i[2]) \geq (\alpha\zeta\lambda)/2$. Now notice that
$\sum_{i=1}^\zeta \vecP^i[1] - \vecP^i[2] \leq \eta$, hence it must be the case
that $\sum_{i =1}^\zeta (\vecP^i[2] - \vecP^i[1])_+ \geq  (\alpha
\zeta\lambda)/2 - \eta$.
Now for the value of $q = 2 \ell^3 T^2/\eta^5$ and if $\alpha \geq
\eta^2/(T\ell)$, the \emph{negative} contribution of ``term 1'' is at least $q
\alpha \zeta \lambda/2$ which greater than the maximum possible  positive
contribution of ``term 2" which is $\ell^2\zeta/\eta$. It is easy to see that
these quantities are equal and hence the total contribution of ``term 1" and
``term 2" together is at most $\eta$. \smallskip \\
\noindent{\bf Case 3} {\it When $|\vecP^i[1] - \vecP^i[2]|$ is ``small" most of
the time.} In this case the parameter $\eta$ is actually well-tuned (which was
not the case when $|\vecP^i[1] - \vecP^i[2]| \approx \ell$) and gives us a small
overall regret. (See Lemma~\ref{lemma:ftpl-refined}.)
We have $\alpha < \eta^2/(T\ell)$.  Note that $\alpha \ell \leq \lambda =
\eta^2/T$ and that $\zeta \leq T /\ell$. In this case ``term 2" can be bounded
easily as follows: $ \frac{\ell}{\eta} \sum_{i=1}^\zeta |\vecP^i[1] -
\vecP^i[2]| \leq \frac{\ell}{\eta} ( \alpha \zeta \ell + (1 - \alpha) \zeta
\lambda) \leq 2 \eta $  \smallskip \\
The above three cases exhaust all possibilities and hence no matter what the
nature of the payoff sequence, the expected regret of $\DFPL$ is bounded by
$O(\eta)$ as required. The expected total communication is easily seen to be
$O(q T + Tk/\ell)$ -- the $q(T/\ell)$ blocks on which \emph{step} $\ftpl$ is
used contribute $O(\ell)$ communication each, and the $(1-q) (T/\ell)$ blocks
where \emph{block} $\ftpl$ is used contributed $O(k)$ communication each.
%
%
\begin{remark} \label{rem:n-experts}
Our algorithm can be generalized to $n$ experts by recursively dividing the set
of experts in two and applying our algorithm to two meta-experts, as shown in
Section~\ref{sec:sp-ub-n-experts} in the Appendix. However, the bound obtained
in Section~\ref{sec:sp-ub-n-experts} is not optimal in terms of the number of
experts, $n$. This observation and Lemma~\ref{lem:DFPL} imply
Theorem~\ref{thm:main}.
\end{remark}

\begin{remark} \label{rem:sp-counter}
The assumption that there is a global counter is necessary because our algorithm
divides the input into blocks of size $\ell$. However, it is not an impediment
because it is sufficient that the block sizes are in the range $[0.99 \ell, 1.01
\ell]$. Assuming that the coordinator always signals the beginning and end of
the block (by a broadcast which only adds $2k$ messages to any block), we can
use a distributed counter that guarantees a very tight approximation to the
number of queries received in each block with at most $O(k \log(\ell))$ messages
communicated (see ~\cite{HYZ12}).
\end{remark}
%
%


\subsection{Lower Bounds}
\label{sec:lower_bound}
%
In this section we give a lower bound on distributed counter algorithms in the
site prediction model. Distributed counters allow tight approximation
guarantees, i.e. for factor $\beta$ additive approximation, the communication
required is only $O(T\log(T)\sqrt{k}/\beta)$~\cite{HYZ12}. We observe that the
noise used by $\ftpl$ is quite large, $O(\sqrt{T})$, and so it is tempting to
find a suitable $\beta$ and run $\ftpl$ using approximate cumulative payoffs. We
consider the class of algorithms such that: \smallskip \\
\mbox{~}(i) Whenever each site receives a query, it has an (approximate) cumulative payoff
of each expert to additive accuracy $\beta$. Furthermore, any communication is
only used to maintain such a counter. \\
\mbox{~}(ii) Any site only uses the (approximate) cumulative payoffs and any local
information it may have to choose an expert when queried. \smallskip \\
However, our negative result shows that even with a highly accurate counter
$\beta = O(k)$, the non-stochasticity of the payoff sequence 
may cause any such algorithm to have $\Omega(\sqrt{kT})$
regret. Furthermore, we show that any distributed algorithm that implements
(approximate) counters to additive error $k/10$ on all sites\footnote{The approximation guarantee is
only required when a site receives a query and has to make a prediction.} is at least $\Omega(T)$. 

\begin{theorem} \label{thm:sp_lb_dc} At any time step $t$, suppose each site has
an (approximate) cumulative payoff count, $\tilde{\vecP}^t[a]$, for every expert
such that $|\vecP^t[a] - \tilde{\vecP}^t[a]| \leq \beta$. Then we have the
following: \smallskip \\
1. If $\beta \leq k$, any algorithm that uses the approximate counts
$\tilde{\vecP}^t[a]$ and \emph{any} local information at the site
making the decision, cannot achieve expected regret asymptotically better than
$\sqrt{\beta T}$. \\
%
2. Any protocol on the distributed system that guarantees that at each time
step, each site has a $\beta = k/10$ approximate cumulative payoff with
probability $\geq 1/2$, uses $\Omega(T)$ communication.\\
%
\end{theorem}


\bsection{Coordinator-prediction model}
\label{sec:cp}
In the co-ordinator prediction model, as mentioned earlier it is possible to use
the label-efficient forecaster,~$\LEF$ (Chap. 6 \cite{CBL06,CBLS05}). Let $C$ be
an upper bound on the total amount of communication we are allowed to use. The
label-efficient predictor translates into the following simple protocol:
Whenever a site receives a payoff vector, it will forward that particular payoff
to the coordinator with probability $p \approx C/T$. The coordinator will always
execute the exponentially weighted forecaster over the sampled subset of payoffs
to make new decisions. Here, the expected regret is $O(T\sqrt{\log(n)/ C})$. In
other words, if our regret needs to be $O(\sqrt T)$, the communication needs to
be linear in $T$.

We observe that in principle there is a possibility of better algorithms in this
setting for mainly two reasons: (i) when the sites send payoff vectors to the
co-ordinator, they can send cumulative payoffs rather than the latest ones, thus
giving more information, and (ii) the sites may decided when to communicate as a
function of the payoff vectors instead of just randomly. However, we
present a lower-bound that shows that for a natural family of algorithms achieving
regret $O(\sqrt{T})$ requires at least $\Omega(T^{1-\epsilon})$ for every
$\epsilon > 0$, {even when $k = 1$}. 
%
The type of algorithms we consider may have an arbitrary communication protocol,
but it satisfies the following: (i) Whenever a site communicates with the
coordinator, the site will report its local cumulative payoff vector. (ii) When
the coordinator makes a decision, it will execute, $\FPL(\sqrt{T})$, (follow the
perturbed leader with noise $\sqrt{T}$) using the latest cumulative payoff
vector. The proof of Theorem \ref{thm:cp_lb_obliv_k} appears in Appendix
\ref{app:cp_lb} and the results could be generalized to other regularizers.
%
%
%
\begin{theorem} \label{thm:cp_lb_obliv_k}
Consider the distributed non-stochastic expert problem in coordinator prediction
model. Any algorithm of the kind described above that achieves regret
$O(\sqrt{T})$ must use $\Omega(T^{1-\epsilon})$ communication against an
oblivious adversary for every constant $\epsilon$.
\end{theorem}

%
%
%


\bsection{Simulations}
\label{sec:simulations}

\begin{figure*}[ht!]
\vskip -6pt
\centering{
$
\begin{array}{ccccc}
\includegraphics[scale=0.41]{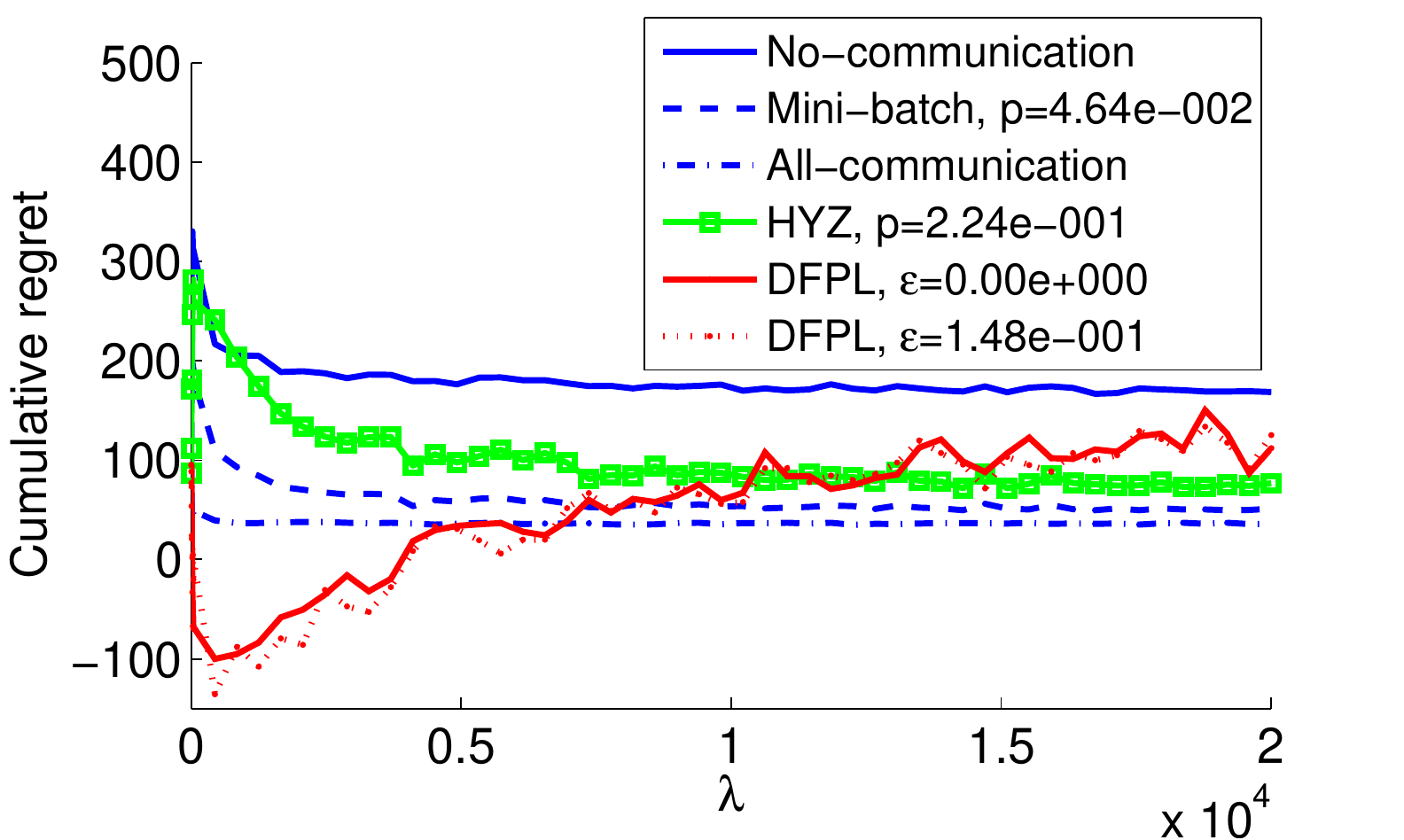} &\mbox{ }&
\includegraphics[scale=0.41]{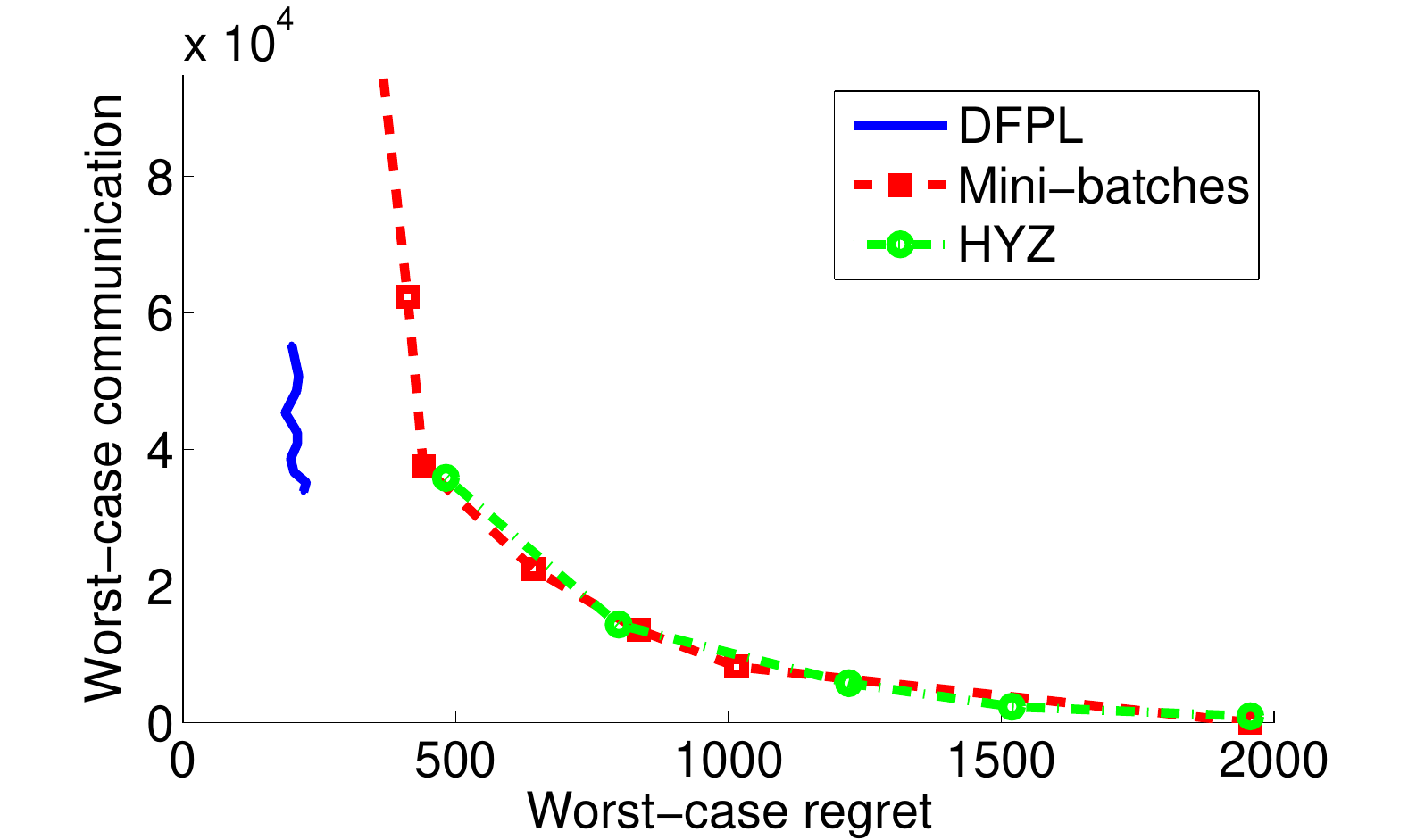} \\
\mbox{(a)} &\mbox{ }& \mbox{(b)}
\end{array}$}
\vskip -6pt
\caption{\small (a) - Cumulative regret for the MC sequences as a function of correlation $\lambda$, (b) - Worst-case cumulative regret vs. communication cost for the MC and zig-zag sequences.}
\label{fig:simulations}
\vskip -6pt
\end{figure*}

In this section, we describe some simulation results comparing the
efficacy of our algorithm $\DFPL$ with some other techniques. We compare $\DFPL$ against
simple algorithms -- \emph{full communication} and \emph{no
communication}, and two other algorithms which we refer to as \textsf{mini-batch}
and $\HYZ$. In the mini-batch algorithm, the coordinator requests randomly, with some
probability $p$ at any time step, all cumulative payoff vectors at all
sites. It then broadcasts the sum (across all of the sites) back to the sites,
so that all sites have the latest cumulative payoff vector. Whenever such a
communication does occur, the cost is $2k$. We refer to this as
\textsf{mini-batch} because it is similar in spirit to the mini-batch algorithms
used in the stochastic optimization problems. In the $\HYZ$ algorithm, we use
the distributed counter technique of Huang et al. \cite{HYZ12} to maintain the
(approximate) cumulative payoff for each expert. Whenever a counter update occurs, the
coordinator must broadcast to all nodes to make sure they have the most current
update.

We consider two types of synthetic sequences. The first is a zig-zag sequence, with $\mu$ being the
length of one increase/decrease. For the first $\mu$ time steps the payoff
vector is always $(1, 0)$ (expert 1 being better), then for the next $2 \mu$
time steps, the payoff vector is $(0, 1)$ (expert 2 is better), and then again
for the next $2\mu$ time-steps, payoff vector is $(1,0)$ and so on.
The zig-zag sequence is also the sequence used in the proof of the lower bound in Theorem~\ref{thm:cp_lb_obliv_k}.
The second is a two-state Markov chain (MC) with states $1, 2$ and $\Pr[1 \to 2] =
\Pr[2 \to 1] = \frac{1}{2\lambda}$. While in state $1$, the payoff vector is
$(1,0)$ and when in state $2$ it is $(0,1)$.

In our simulations we use $T=20000$ predictions, and $k=20$ sites. Fig.
\ref{fig:simulations} (a) shows the performance of the above algorithms for the
MC sequences, the results are averaged across $100$ runs, over both the
randomness of the MC and the algorithms. Fig. \ref{fig:simulations} (b) shows
the worst-case cumulative communication vs the worst-case cumulative regret trade-off for three algorithms: $\DFPL$,
\textsf{mini-batch} and $\HYZ$, over all the described sequences. While in general it is hard to compare
algorithms on non-stochastic inputs, our results confirm that for non-stochastic
sequences inspired by the lower-bounds in the paper, our algorithm $\DFPL$
outperforms other related techniques.

\bibliographystyle{unsrt}
\bibliography{ref}

\newpage

\appendix
\bsection{Adaptive Adversary}
\label{app:adaptive}
This section contains a proof of Theorem~\ref{thm:adaptive}. The proof makes use
of Khinchine's inequality (see Appendix A.1.14 in \cite{CBL06}).

\begin{theoremK} Let $\sigma_1, \ldots, \sigma_n$ be Rademacher random
variables, i.e.  $\Pr[\sigma_i = 1] = \Pr[\sigma_i = -1] = 1/2$. Then for any
real numbers $a_1, \ldots, a_n$, 
\[ \E\left[| \sum_{i=1}^n a_i \sigma_i|\right] \geq \frac{1}{\sqrt{2}} \sqrt{\sum_{i=1}^n
a_i^2}  = \frac{1}{\sqrt{2}} \sqrt{\E\left[\left(\sum_{i=1}^n a_i
\sigma_i\right)^2\right]} \]
\end{theoremK}


\begin{proof}[Proof of Theorem~\ref{thm:adaptive}]
The adaptive adversary divides the total $T$ time steps into $T/k$ time blocks,
each consisting of $k$ time-steps. During each block of $k$ time-steps, each
of the $k$ sites receives exactly $1$ query.
At time $t = 1, k+1, 2k+1, \ldots$, the
adversary tosses an unbiased coin. Let $\vecp_H$ denote the payoff vector
corresponding to \emph{heads}, where $\vecp_H[1] = 1$ and $\vecp_H[2]= 0$.
Similarly let $\vecp_T$ (corresponding to \emph{tails}) be such that $\vecp_T[1]
= 0$ and $\vecp_T[2] = 1$. For $i = 1, \ldots, T/k$ and $j = 1, \ldots, k$, the
\emph{adaptive adversary} does the following: At time $(i-1)k + j$, if there was
no communication on part of the decision maker (distributed system) between time
steps $(i-1)k + 1, \ldots, (i-1)k + j-1$ -- then if the coin toss at time
$(i-1)k + 1$ was heads the payoff vector is $\vecp_H$, otherwise it is
$\vecp_T$. On the other hand if there was any communication, then the adaptive
adversary tosses a random coin and sets the payoff vector accordingly. 

Consider the expected payoff of the algorithm: At time $t = (i-1) k + j$, if
there was communication between time steps $(i-1)k + 1$ to $(i-1)k + j-1$, then
the adversary has chosen the payoff vector uniformly at random between $\vecp_H$
and $\vecp_T$ and hence the expected reward at time step $t$ is exactly $1/2$.
On the other hand if there was no communication between these time steps, then
the site $j$ making the decision has no information about the coin toss of the
adversary at time $(i-1)j + 1$, and hence the expected reward is still $1/2$.
Thus, the total expected reward of the algorithm (by linearity of expectation)
is $T/2$. 

Note that,
\begin{align}
\E\left[\max_{i= 1, 2} \sum_{t=1}^T \vecp^t[i]\right] &= \frac{1}{2}
\left(\E\left[\sum_{t=1}^T \vecpt[1] + \vecpt[2]\right] + \E\left[ |\sum_{t=1}^T
(\vecp^t[1] - \vecp^t[2])|\right] \right) \nonumber \\
& = \frac{T}{2} + \frac{1}{2} \E\left[|\sum_{t=1}^T (\vecp^t[1] -
\vecp^t[2])|\right]
\label{eqn:beforeKhinchine} \\
%
\intertext{
Let $I \subseteq [T/k]$ be the indices of the blocks for which there was some
communication. Consider blocks in $I$ and those outside of $I$. Suppose the
block $(i-1)k + 1, \ldots, ik$ is such that $i \not\in I$, then
$|\sum_{t=(i-1)k+1}^{t=ik} \vecpt[1] - \vecpt[2]| = k$. Note that all such block
sums (as random variables) are independent of all other block sums. For some
block $(i-1)k + 1, \ldots, ik$ such that $i \in I$, let $c(i)$ be such the first
such that communication occurs at block $(i-1)k + c(i)$. Then
$|\sum_{t=(i-1)k+1}^{t=(i-1)k+c(i)} \vecpt[1] - \vecpt[2]| = c(i)$, also note
that $\vecpt$ for $t = (i-1)k + c(i) +1, \ldots, ik$ are all based on
independent coin tosses. Then note that, }
\sum_{t=1}^T \vecpt[1] - \vecpt[2] &= \sum_{i \not\in I} k \sigma_i,1 + \sum_{i \in
I} (c(i) \sigma_i,1 + \sum_{j=c(i)+1}^k \sigma_{i, j}),
\intertext{where $\sigma_i,j$ are the Rademacher variables corresponding to the
coin tosses of the adversary at time step $(i-1)k + j$. Also note that,}
\E\left[\left(\sum_{t=1}^T \vecpt[1] - \vecpt[2]\right)^2 \right] &\geq \left(
\frac{T}{k} - |I|\right)k^2 \nonumber
\intertext{Then, Khinchine's inequality and (\ref{eqn:beforeKhinchine}) gives us
that}
\E[\max_{i = 1, 2} \sum_{t=1}^T \vecp^t[i]] &\geq \frac{T}{2} +
\frac{1}{2\sqrt{2}} \sqrt{\E\left[\left(\sum_{t=1}^T \vecpt[1] -
\vecpt[2]\right)^2 \right]} \nonumber \\
&\geq \frac{T}{2} + \frac{1}{2 \sqrt{2}} \sqrt{\left(\frac{T}{k} - |I|\right)
k^2} \nonumber
\end{align}
Now, unless $|I| = (T/k)(1 - o(1))$, it must be the case that $\E[\max_{i=1,2}
\sum_{t=1}^T \vecpt[i]] \geq T/2 + \Omega(\sqrt{kT})$ leading to total expected
regret $\Omega(\sqrt{kT})$. Hence, any algorithm that achieves regret
$o(\sqrt{kT})$ must have communication $(1 - o(1)) T/k$.
\end{proof}

\bsection{Follow the Perturbed Leader}
\label{app:ftpl-proof}
\begin{proof}[Proof of Lemma~\ref{lemma:ftpl-refined}]
We first note that using the given notation, the regret guarantee of
$\FPL(\eta)$ (see Fig. \ref{alg:DFPL}(b)) is
\begin{align*}
\E[R] &\leq \frac{B}{\eta}\sum_{t=1}^T |\vecpt|_1 + \eta
\end{align*}
The above appears in the analysis of Kalai and Vempala ~\cite{KV05}. Note that
although $|\vecpt|_1 = \vecpt[1] + \vecpt[2]$ ($\vecpt[a] \geq 0$ in our
setting), we can use the following trick. We first observe that since
$\FPL(\eta)$ only depends on the difference between the cumulative payoffs of
the two experts, we may replace the payoff vectors $\vecpt$ by
$\tilde{\vecp}^t$, where \smallskip \\
(i) if $\vecpt[1] \geq \vecpt[2]$, $\tilde{\vecp}^t[1] = \vecpt[1] - \vecpt[2]$
and $\tilde{\vecp}^t[2] = 0$
(ii) if $\vecpt[1] < \vecpt[2]$, $\tilde{\vecp}^t[1] = 0$
and $\tilde{\vecp}^t[2] = \vecpt[2] - \vecpt[1]$

Next, we observe that the regret of $\FPL(\eta)$ with payoff sequence $\vecpt$
and $\tilde{\vecp}^t$ is identically distributed, since the random choices only
depend on the difference between the cumulative payoffs at any time. Lastly, we
note that $|\tilde{\vecp}^t|_1 = |\vecpt[1] - \vecpt[2]|$, which completes the
proof.
\end{proof}

\bsection{Site Prediction : Missing Proofs}
\subsection{Generalizing $\DFPL$ to $n$ experts}
\label{sec:sp-ub-n-experts}

In this section, we generalize our $\DFPL$ algorithm for two experts to handle
$n$ experts. Lemma~\ref{lem:DFPL} showed that algorithm $\DFPL$, in the
setting of two experts, guarantees that the expected regret is at most $c_0
\sqrt{\ell^{5/6} T}$, where $c_0$ is a universal constant.

Our generalization follows a recursive approach. Suppose that some algorithm
$\alg$ can achieve expected regret, $c_0 \log(n) \sqrt{\ell^{5/6} T}$ with $n$
experts, we show that we can construct algorithm $\alg'$ that achieves expected
regret, $c_0 (\log(n) + 1)$ with $2n$ experts as follows: We run 2 independent
copies of $\alg$ (say $\alg_1$ and $\alg_2$) such that $\alg_1$ only deals with
the first $n$ experts $a_1, a_2, ..., a_n$ and $\alg_2$ with the rest of the
experts $a_{n + 1}, ..., a_{2n}$. Then our algorithm $\alg'$ treats $\alg_1$ and
$\alg_2$ as 2 experts and runs the $\DFPL$ algorithm (Section
\ref{sec:upper_bound}) over these two experts. The analysis for regret is
straightforward: 
%

Let the regret for $\alg_1$ be $R_1$ and the regret for $\alg_2$ be $R_2$. We
have $$\E[\payoff(\alg_1)] \geq \max_{i \in [n]}\sum_{t \leq T}\vecp^t[i] -
\E[R_1] \quad \mbox{ and } \quad \E[\payoff(\alg_2)] \geq \max_{i \in \{n + 1,
..., 2n\}}\sum_{t \leq T}\vecp^t[i] - \E[R_2].$$ We know that $\E[R_1] \leq c_0
\log(n) \sqrt{\ell^{5/6} T}$ and $\E[R_2] \leq c_0 \log(n) \sqrt{\ell^{5/6} T}$.

Next, we can see that
\begin{align*}
\E[\payoff(\alg') \mid \payoff(\alg_1), \payoff(\alg_2)] &\geq
\max\{\payoff(\alg_1), \payoff(\alg_2)\} - c_0 \sqrt{\ell^{5/6}T} \\
\intertext{We can use the above expression to conclude (taking expectations)
that}
\E[\payoff(\alg')] &\geq \E[\payoff(\alg_1)] - c_0\sqrt{\ell^{5/6} T}  \\
\E[\payoff(\alg')] &\geq \E[\payoff(\alg_2)] - c_0\sqrt{\ell^{5/6} T}
\intertext{But using the above two inequalities we can conclude that}
\E[\payoff(\alg')] &\leq \max_{i \in [2n]} \sum_{t \leq T} \vecp^t[i] -
c_0(\log(n) + 1) \sqrt{l^{5/6} T}
%
%
\end{align*}

This immediately implies that for $n$ experts (starting from base case of $n =
2$ where $\DFPL$ works), this recursive approach results in an algorithm for $n$
experts achieves regret $O(\log(n)\sqrt{\ell^{5/6}T})$. In order to analyze the
communication, we observe that in order to implement the algorithm correctly,
when algorithm (which is $\DFPL$ at some depth in the recursion) decides to
communicate at each time step on a block, the communication on that block is
$\ell$. There are at most $n$ copies of $\DFPL$ running (depth of the recursion
is $\log(n) - 1$).  However, the corresponding term in the communication bound
$O( n q T \ell)$ is lower than the term arising from blocks where communication
occurs only at the beginning and end of block, $O((1-qn) T k/\ell)$. Thus, the
expected communication (in terms of number of messages) is \emph{asymptotically}
the same as in the case of $2$ experts. If we count communication complexity as
the cost of sending $1$ real number, instead of one message, then the total
communication cost is $O(nTk/\ell)$. 

%
%
%
%
%

\subsection{Lower Bounds}
\subsubsection{No Communication Protocol}
\label{app:sp-lb-zero}

In the site-prediction setting, we show that any algorithm that uses no
communication must achieve regret $\Omega(\sqrt{kT})$ on some sequence. The
proof is quite simple, but does not follow directly from the $\Omega(\sqrt{T})$
lower-bound of the non-distributed case, because although the $k$ sites each run
a copy of some $\FPL$-like algorithm, the best expert might be different across
the sites. We only consider the case when $n = 2$, since we are more interested
in dependence on $T$ and $k$. 
%
%
%

\begin{lemma}\label{lemma:sp-no-comm} If no-communication protocol is used in
the site-prediction model expected regret achieved by any algorithm is at least
$\Omega(\sqrt{kT})$. 
\end{lemma}
\begin{proof}
The oblivious adversary does the following: Divide $T$ time steps into $T/k$ blocks of
size $k$. For each block, toss a random coin and set the payoff vector to be
$\vecp_H = (1, 0)$ for heads or $\vecp_T = (0, 1)$ for tails. And each query in
a block is assigned to one site (say in a cyclic fashion). Note that the
expected reward of any algorithm that does not use any communication is $T/2$.
Because, no site at any time can perform better than random guessing. But the
standard analysis shows that for the sequence as constructed above
$\E[\max_{a=1, 2} \sum_{t=1}^T \vecpt[a]] \geq T/2 + \Omega(k \sqrt{T/k}) = T/2
+ \Omega(\sqrt{kT})$.
\end{proof}

\subsubsection{Lower Bound using Distributed Counter}

This section contains proof of Theorem~\ref{thm:sp_lb_dc}.

\begin{proof}[Proof of Theorem~\ref{thm:sp_lb_dc}]  \mbox{ }\\

\noindent{\bf Part 1}: The \emph{oblivious} adversary decides to only use
$\beta$ out of the $k$ sites. The adversary divides the input sequence into
$T/\beta$ blocks, each block of size $\beta$. For each block, the adversary
tosses an unbiased coin and sets the payoff vector $\vecp_H = (1, 0)$ or
$\vecp_T = (0, 1)$ according to whether the coin toss resulted in heads or
tails. Let $\tilde{\vecP}^t[a] = \vecP^{t^*}[a]$, where $t^*$ is largest such that
$t^* < t$ and $t^* = \beta i$ for some integer $i$ (i.e. $t^*$ is the time at
the end of the block). Note that $|\tilde{\vecP}^t[a] - \vecP^t[a]| \leq \beta$,
so $\tilde{\vecP}^t[a]$ is a valid (approximate) value of the cumulative payoff of
action $a$. However, since the payoff vectors across the blocks are completely
uncorrelated and each site makes a decision only once in each block, the
expected reward at any time step $t$ is $1/2$, and overall expected reward is
$T/2$.

Note, that it is easy to show that $\E[\max_{i =1, 2} \sum_{t=1}^T  \vecp^t[i]]
\geq T/2 + \Omega(\sqrt{\beta T})$ using standard techniques. Thus the expected
regret is at least $\Omega(\sqrt{\beta T})$.  \medskip \\

\noindent {\bf Part 2}: Let $\beta = k/10$. Now consider the input sequence that
is all $1$. But that this is divided into $T/k$ blocks of size $k$. For each
block, the \emph{oblivious} adversary chooses a random permutation of $\{1,
\ldots, k\}$ and allocates the $1$ to the site in that order. Note that when the
site receives a $1$, it is required to have an $\beta$-approximate value to the
current count.  Suppose there was no communication since this site last received
a query, then at that time the estimate at this site was at most $ik + \beta$.
Now, depending on where in the permutation the site is it may be required to
have a value in any of the intervals $[ik - \beta, ik + \beta], [ik, ik+2\beta],
[ik+\beta, ik+3\beta], \ldots, [(i+1)k -\beta, (i+1)k+ 2\beta]$. There are at
least $5$ disjoint intervals in this state and each of them are equally
probable. Thus with probability at least $4/5$, in the absence of any
communication, this site fails to have the correct approximate estimate.

If on the other hand, every site does communicate at least once every time it
receives a query. The total communication is at least $T$.
\end{proof}

\bsection{Proof of Theorem~\ref{thm:cp_lb_obliv_k}}
\label{app:cp_lb}
\begin{proof}[Proof of Theorem~\ref{thm:cp_lb_obliv_k}]
To prove Theorem~\ref{thm:cp_lb_obliv_k}, we construct a set of reward sequences
$\vecp^t_0, \vecp^t_1, ..., $, and show that any FPL-like algorithm (as
described in Section~\ref{sec:cp}), will have regret $\Omega(\sqrt{T})$ on least
one of these sequences unless the communication is essentially linear in $T$.

Before we start the actual analysis, we need to introduce some more notation.
First, recall that $C$ is an upper bound on the amount of communication allowed
in the protocol. We shall focus reward sequences where at any time-step exactly
one of the experts receives payoff $1$ and the other expert receives payoff $0$,
i.e. $\vecp^t \in \{(0, 1), (1, 0)\}$ for any $t$.  Let $g^{\vecp}(t) =
\vecp^t[1] - \vecp^t[2]$,  and let $G^{\vecp}(t) = \sum_{i = 1}^t g^{\vecp}(t)$.
Thus, we note that the payoff vectors $\vecp$, the function $g^{\vecp}$, and the
function $G^{\vecp}$ all encode equivalent information regarding payoffs as a
function of time.

Suppose, $\alg$ is an algorithm that achieves optimal regret under the
communication bound $C$. Let $r$ denote the random coin tosses used by, $\alg$.
Thus we may think of $r$ as being a string of length $\mathrm{poly}(n, k) T$
fixed ahead of time.  Let $\vecp^1$, ..., $\vecp^T$ be a specific input
sequence. Let $T_1, T_2, \ldots, T_C$ denote the time-steps when communication
occurs. We note that $T_i$ may depend on $r_i$ which is a prefix of the (random)
string $r$, which the algorithm observes until time-step $T_i$ and may also
depend on the payoff vectors $\vecp^1, \ldots, \vecp^{T_i}$. 


Next, we describe the set of reward sequences to ``fool'' the algorithm. Let
$\lambda$ be a parameter that will be fixed later. We construct up to $(T/(2
\lambda)) + 1$ possible payoff sequences. We denote this payoff sequences as
$\vecp_{(0)}, \vecp_{(1)}, \ldots, \vecp_{(T/(2\lambda)) + 1}$. These sequences
are constructed as follows:
\begin{itemize}
\item $\vecp_{(0)}$: Let $g^+$ denote a sequence of $\lambda$ consecutive $1$'s
and $g^-$ denote a sequence of $\lambda$ consecutive $-1$'s. Then the sequence
$\langle g^{\vecp_{(0)}}(t) \rangle_{t \leq T}$ is defined to be the sequence
$g^-, g^+, g^+, g^-, g^-, ...$, i.e. $g^{\vecp_{(0)}}(t) = -1$ if $\lceil (t -
1)/\lambda \rceil$ is even and $g^{\vecp_{(0)}}(t) = 1$ if $\lceil (t -
1)/\lambda \rceil$ is odd. Furthermore, we assume that $T = (4 m_1 + 3)\lambda$
for some integer $m_1$. This means that $G^{\vecp_{(0)}}(T) = \lambda$, i.e.
eventually expert $1$ will be the better expert.
\item $\vecp_{(i)}$ for $i > 0$ and $i$ even: In this payoff sequence, the
payoff vectors for the first $(2i - 1) \lambda$ time-steps will be identical to
those in $\vecp_{0}$. For the rest of the time-steps the payoff vector will
always be $\{(1, 0)\}$, i.e. the first expert always receives a unit payoff for
$t > (2i-1)\lambda$. Thus, for sequences of this form, where $i$ is even, expert
$1$ will be the better expert.
\item $\vecp_{(i)}$ for $i > 0$ and $i$ odd: In this payoff sequence, the payoff
vectors for the first $(2i - 1)\lambda$ time-steps will be identical to
$\vecp_{(0)}$. For the rest of the time-steps, the payoff vector will always be
$\{(0, 1)\}$, i.e. the second expert always receives a unit payoff after $t >
(2i - 1)\lambda$. Thus, for sequences of this form, where $i$ is odd, expert $2$
will be the better expert.
\end{itemize}
Furthermore, in what follows, we assume that there is only one site node. (This
is not a problem, since worst adversary could send all the payoff vectors to
just one of the site nodes.) We shall refer to the $i$-th cycle of the input in
the above sequences as the input between time steps  $(4i+2)\lambda -
(\sqrt{T}/2) + 1$ and $(4i + 4)\lambda + (\sqrt{T}/2)$. Let $F^i$ be an
indicator random variable (depending on the randomness $r$ of the algorithm),
such that $F^i = 0$, if there is \emph{some} communication between the time
steps $2i \lambda + \sqrt{T}/2$ and $(2i+2)\lambda  - \sqrt{t}/2$. If there is
no communication, we will set $F^i = 1$. 

Now, we prove the main result using a series of claims. First, we show add a few
extra communication points, showing that this only increases the payoff of the
algorithm (hence decreases regret). Let ${\mathcal I} = \{ i ~|~ F^{2i} =
F^{2i+1} = F^{2i+2} = 0 \}$. Note that ${\mathcal I}$ itself is a random
variable. For every $i \in {\mathcal I}$, we allow extra communication to the
algorithm (for free) at the end of the following time-steps: $(4i+2) \lambda -
\sqrt{T}/2$ $(4i+2) \lambda + \sqrt{T}/2$, $(4i+4)\lambda - \sqrt{T}/2$, and
$(4i+4)\sqrt{T}/2$. Note, that this extra communication can only increase the
payoff, precisely because $F^{2i} = F^{2i+1} = F^{2i+2} = 0$. This extra
communication is given for free, thus this is favorable to the trade-off of the
algorithm. Despite this we will show that even the regret of this algorithm has
to be large. This is done by a series of claims. Each of which are proved as
lemmas subsequently.

\begin{enumerate}
\item[] {\bf Claim A} Let $R^{\vecp_{(i)}}_{\alg}(1, T)$ denote the (random
variable) regret of playing according to algorithm, $\alg$, against payoff
sequence, $\vecp_{(i)}$ using randomness $r$, between time-steps $1$ and $T$.
Then, if $\E[R^{\vecp_{(i)}}_{\alg}(1, T)] = O(\sqrt{T})$ for all $1 \leq i \leq
T/(2\lambda)$, then $\E[|{\mathcal I}|] \geq \frac{T}{4\lambda}$. This fact is
proved in Lemma~\ref{lem:one}.
\item[] {\bf Claim B} Suppose, $i \in {\mathcal I}$, and let $C(i)$ be the
communication during the $i\th$ cycle. Then we can state the following regarding
the payoff on the rounds with respect to sequence $\vecp_{(0)}$ within the
$i\th$ cycle. Here $c_0$ is some absolute constant.
\[ \payoff^{\vecp_{(0)}}_{\alg}((4i+2)\lambda - \sqrt{T}/2 +1, (4i+4)\lambda +
\sqrt{T}/2) \leq \lambda + \sqrt{T}/2 - \frac{c_0 \sqrt{T}}{C(i)} \]
This fact is proved in Lemma~\ref{lem:two}.
\item[] {\bf Claim C} Let $t$ be a point such that communication happened just
after time step $t$. Let $\tau > t$ be a point such that $G(\tau) = G(t)$. Then
$\payoff^{\vecp_{(0)}}_{\alg}(t+1, \tau)\leq (\tau - t)/2$. This fact is proved
in Lemma~\ref{lem:three}.
\end{enumerate}
Now, let us calculate the regret of the algorithm. If the expected regret of the
algorithm with respect to sequence $\vecp_{(i)}$ for $i > 0$, is at most
$O(\sqrt{T})$, then it must be the case that $\E[|{\mathcal I}|] \geq
T/(4\lambda)$ (using {\bf Claim A} above). Now, we assumed that in the sequence
$\vecp_{(0)}$, expert $1$ eventually wins. Let ${\mathcal I}  = \{i_1, \ldots,
i_k \}$, where $i_1 < i_2 < \cdots < i_k$ and $\E[k] \geq T/(4\lambda)$. Then,
we add up the payoff of the algorithm as follows. First, (using {\bf Claim B} above)
notice that:
\begin{align}
\E[\payoff^{\vecp_{(0)}}_{\alg}((4i_j + 2)\lambda - \sqrt{T}/2 + 1, (4i_j +4
)\lambda + \sqrt{T}/2)] &\leq \lambda + \sqrt{T}_2  - \frac{c_0 \sqrt{T}}{C(i)}
\label{eqn:cp-app-one}
\end{align}
Then let $B_j$ denote the interval, $((4i_j + 4)\lambda + \sqrt{T}/2 + 1,
(4i_{j+1} + 2)\lambda - \sqrt{T}/2)$, \ie between the $i\th$ and the $j\th$
cycle. Also, let $B_0$ denote $(\sqrt{T}/2 +1, (4i_1 + 2)\lambda - \sqrt{T}/2)$
be the interval before the first cycle in ${\mathcal I}$, and let $B_{k} =
((4i_k + 4)\lambda + \sqrt{T}/2 + 1, T - \lambda - \sqrt{T}/2)$ denote the
interval after the last cycle. Now, using {\bf Claim C} above, we get that the payoff
received by algorithms in any interval $B_j$ is half the length of the interval.
Thus, the only time-steps that we have not accounted for is $(1, \sqrt{T}/2)$
and $(T - \lambda - \sqrt{T}/2 + 1, T)$. The total number of time-steps in these
two intervals is $\lambda$. Let us give the algorithm payoff $\lambda$ for free
on these time steps. Then, adding up everything and the payoff of the algorithm,
$\payoff^{\vecp_{(0)}}_{\alg}$ is a random variable defined over the space
measurable by $\{F^i\}_{i \geq 0}$ and $C$

\begin{align}
\payoff^{\vecp_{(0)}}_{\alg}(1, T) &\leq \frac{T}{2} + \frac{\lambda}{2} -
\sum_{j=1}^k \frac{c_0 \sqrt{T}}{C(i_j)} \nonumber 
\end{align}

Thus, we get
\begin{eqnarray*}
\E[R_{\alg}^{\vecp_{(0)}}\mid \{F^i\}_{i \geq 0}, C] & \geq & \E\left[\sum_{i
\in \mathcal I}\frac{\sqrt T}{C(i)}\mid \{F^i\}_{i \geq 0}, C\right] -
\frac{\lambda}{2} \quad \mbox{($\mathcal I$ is measurable by $\{F^i\}_{i \geq 1}$)} \\
& \geq & \E\left[\frac{|\mathcal I|^2 \sqrt T}{C}\mid \{F^i\}_{i \geq 0},
C\right] -\frac{\lambda}{2}\\
& \geq & c_0\frac{|\mathcal I|^2\sqrt T}{C} - \frac{\lambda}{2} \quad \mbox{($\mathcal I$ is measurable by $\{F^i\}_{i \geq 0}$)}
\end{eqnarray*}

We use Jensen's inequality and the fact that $C \geq \sum_{i \in {\mathcal I}}
C(i)$ to get the last inequality. Finally, using {\bf Claim A} and by setting
$\lambda$ appropriately, we get 
\[ \E[R^{\vecp_{(0)}}_{\alg}(1, T)] \geq c_0 T^{1.5
- 2\epsilon_1}{16C} \]
\end{proof}

We now prove the Lemmas mentioned in the above proof. 

\begin{lemma} \label{lem:one} If $\E[R^{\vecp_{(i)}}_{\alg}(1, T)] =
O(\sqrt{T})$ for all $1 \leq i \leq \frac{T}{2\lambda}$, then $\E[|{\mathcal
I}|] \geq \frac{T}{4 \lambda}$. 
\end{lemma}
\begin{proof}
Our crucial observation here is that when the random tosses of the algorithm is
fixed, the algorithm will have identical behavior against the reward sequences
$\vecp_{(0)}$ and $\vecp_{(m)}$ for any $1 \leq m \leq \frac{T}{2\lambda}$ up to
time $2m\lambda - \lambda$. Thus, if we couple the process for executing $\alg$
against $\vecp_{(0)}$ with the one for executing $\alg$ against $\vecp_{(m)}$
with the same random tosses in the algorithm, we are able to relate the random
variables $\{F^i\}_{i \geq 0}$ with the regrets for other reward sequences.
Specifically, it is not difficult to see that
\begin{equation}\label{eqn:even}
\E[R_{\alg}^{\vecp_{(m)}}(1, 2m\lambda+1)\mid \{F^i\}_{i \geq 0}] \geq c_0
\max_{i \mbox{ odd}}\left\{(1-F^i)F^{m - 1}\left(\prod_{j = i + 1}^{m -
2}F^j\right)\right\}\cdot \lambda
\end{equation}
when $m$ is odd and
\begin{equation}\label{eqn:odd}
\E[R_{\alg}^{\vecp_{(m)}}(1, 2m\lambda+1)\mid \{F^i\}_{i \geq 0}] \geq c_0
\max_{i \mbox{ even}}\left\{(1-F^i)F^{m - 1}\left(\prod_{j = i + 1}^{m -
2}F^j\right)\right\}\cdot \lambda
\end{equation}
when $m$ is even.

We may then use this observation to prove Lemma~\ref{lem:two}. Let $m$ be an arbitrary number. We shall show that
$\Pr[m \in \mathcal I] \geq \frac 1 2$.

Let us define the event $\mathcal E(s)$ be the event so that the suffix of $\{F^i\}_{1 \leq i \leq m}$ is $s$. For example,
$\mathcal E(000)$ represents the event that $F^{m - 2} = F^{m - 1} = F^m = 0$. Let partition the probability space into the following events:
{\small
$$\mathcal E(000),
\mathcal E(001),
\mathcal E(010),
\mathcal E(011),
\mathcal E(0100),
\mathcal E(01100),
\mathcal E(11100),
\mathcal E(101),
\mathcal E(0110),
\mathcal E(1110), \mbox{ and }
\mathcal E(111).$$}
Furthermore, we let $\mathcal E_0(01100)$ be the subset of $\mathcal E(01100)$ such that
the last zero in the sequence $F^0, ..., F^{m - 5}$ has an even index. And let $\mathcal E_1(01100) = \mathcal E(01100) - \mathcal E_0(01100)$.
Similarly, we let
\begin{itemize}
\item $\mathcal E_0(1110)$ be the subset of $\mathcal E(1110)$ such that
the last zero in the sequence $F^0, ..., F^{m - 4}$ has an even index; let $\mathcal E_1(1110) = \mathcal E(1110) - \mathcal E_0(1110)$
\item $\mathcal E_0(111)$ be the subset of $\mathcal E(111)$ such that
the last zero in the sequence $F^0, ..., F^{m - 3}$ has an even index; let $\mathcal E_1(111) = \mathcal E(111) - \mathcal E_0(111)$
\end{itemize}
Now the whole probability space can be partitioned into the following events: $\mathcal E(000)$,
$\mathcal E(001)$,
$\mathcal E(010)$,
$\mathcal E(011)$,
$\mathcal E(0100)$,
$\mathcal E(01100)$,
$\mathcal E_0(11100),\mathcal E_1(11100)$
$\mathcal E(101)$,
$\mathcal E(0110)$,
$\mathcal E_0(1110), \mathcal E_1(1110)$
$\mathcal E_0(111), \mathcal E_1(111)$.

Let $\epsilon_2$ be an arbitrary constant such that $0 < \epsilon_2 <
\epsilon_1$.  It is not difficult to see that if any of the events above, except
for $\mathcal E(000)$, happens with probability at least $T^{-\epsilon_2}$, then
one of $\vecp_i$ will have $\omega(\sqrt T)$ regret. We will just examine one
event to illustrate the idea. The rest of them can be verified in a similar way.
Suppose $\Pr[\mathcal E(001)] \geq T^{-\epsilon_2}$, we have \begin{eqnarray*}
\E[R_{\alg}^{\vecp_{m - 1}}(1,T)] & \geq & \E[R_{\alg}^{\vecp_{m - 1}}(1,T)\mid
\mathcal E(001)]\Pr[\mathcal E(001)] \\ & \geq &  \E[R_{\alg}^{\vecp_{m -
1}}(1,T)\mid \mathcal E(001)]\Pr[\mathcal E(001)]  \\ & = & \omega(\sqrt T)
\quad \mbox{(By (\ref{eqn:even}) and (\ref{eqn:odd}))}.  \end{eqnarray*} Thus,
we can conclude that $\Pr[\mathcal E(000)] \geq 1 - 13T^{-\epsilon_2} \geq \frac
1 2$ for sufficiently large $T$, which concludes our proof.
\end{proof}

\begin{lemma} Let $i \in {\mathcal I}$, and let $C(i)$ denote the communication
in the $i\th$ cycle. Then,
\[ \E[\payoff^{\vecp_{(0)}}_{\alg}((4i+2)\lambda- \sqrt{T}/2 + 1, (4i+4)\lambda +
\sqrt{T}/2)] \leq \lambda + \sqrt{T}/2 - \frac{c_0 \sqrt{T}}{C(i)} \]
\label{lem:two}
\end{lemma}
\begin{proof}
Actually, using Lemma~\ref{lem:three} it is easy to see that
$\E[\payoff^{\vecp_{(0)}}_{\alg}((4i+2)\lambda + \sqrt{T}/2 + 1,(4i+4)\lambda -
\sqrt{T}/2)] \leq \lambda - \sqrt{T}/2$. Now, let us consider the interval,
$((4i+2)\lambda - \sqrt{T}/2 +1, (4i+2)\lambda + \sqrt{T}/2)$. Let $T_0 =
(4i+2)\lambda - \sqrt{T}/2, T_1, \ldots, T_c = (4i+2)\lambda + \sqrt{T}/2$, be
the time-steps when communication occurs. Note that the communication at
time-steps $T_0$ and $T_c$ is for free, and that $c \leq C(i)$. Let $w(x)$
denote the probability of picking the first expert according to follow the
perturbed leader ($\FPL$), if the $x$ is the difference between the cumulative payoff of
the first and second expert so far. Thus, if $x = -\sqrt{T}$, $w(x) = 0$ and if
$x = \sqrt{T}$, $w(x) = 1$. We have, 
\begin{align*} w(x) = \begin{cases} 1 & x > \sqrt{T} \\ 1 - \frac{1}{2} \left( 1 -
\frac{x}{\sqrt{T}} \right)^2 & 0 \leq x \leq \sqrt{T} \\ \frac{1}{2} \left(1
+ \frac{x}{\sqrt{T}}\right)^2 & -\sqrt{T} \leq x \leq 0 \\ 0 & x < - \sqrt{T}
\end{cases}\end{align*}
Then, we have
\begin{align*}
\E[\payoff^{\vecp_{(0)}}_{\alg}((4i+2)\lambda - \sqrt{T}/2 +1, (4i+2)\lambda +
\sqrt{T}/2)] &= \sum_{j=0}^{c-1} w(G^{\vecp_{(0)}}(T_j)) (T_{j+1} - T_j)
\end{align*}

We use the following claim (which is an exercise in simple calculus) to complete
the proof. 

\begin{claim} Let $f : [a, b] \rightarrow \reals^+$ be an increasing function
such that $f^\prime(x) \geq L$ on $[a, b]$. Let $x_0 = a < x_1 < \cdots x_c =
b$, then 
\[ \sum_{j=0}^{c-1} f(x_j)(x_{j+1} - x_j)  \leq \int_{a}^b f(x) dx -
\frac{L(b-a)^2}{c} \]
\end{claim}

Now, notice that $G^{\vecp_{(0)}}(T_0) = -\sqrt{T}/2$, $G^{\vecp^{(0)}}(T_c) =
\sqrt{T}/2$, and $\int_{-\sqrt{T}/2}^{\sqrt{T}/2} w(x) dx = \sqrt{T}/2$. Also,
$w^\prime(x) \geq 1/(2\sqrt{T})$. Thus, applying the above claim, we get
\begin{align*}
\E[\payoff^{\vecp_{(0)}}_{\alg}((4i+2)\lambda - \sqrt{T}/2 +1, (4i+2)\lambda +
\sqrt{T}/2)] &= \sum_{j=0}^{c-1} w(G^{\vecp_{(0)}}(T_j)) (T_{j+1} - T_j)
&\leq \sqrt{T}/2 - \frac{c_0 \sqrt{T}}{C(i)}
\end{align*}
Similarly, we can prove that.
\begin{align*}
\E[\payoff^{\vecp_{(0)}}_{\alg}((4i+4)\lambda - \sqrt{T}/2 +1, (4i+4)\lambda +
\sqrt{T}/2)] &= \sum_{j=0}^{c-1} w(G^{\vecp_{(0)}}(T_j)) (T_{j+1} - T_j)
&\leq \sqrt{T}/2 - \frac{c_0 \sqrt{T}}{C(i)}
\end{align*}
Adding up across the three intervals, we can complete the proof the lemma.
\end{proof}

Finally, we prove the following:

\begin{lemma} \label{lem:three} Let $\{T_i\}_{i \geq 1}$ be point where
communication occurs in the algorithm $\alg$. Pick some $T_i$ and let $\tau >
T_i$, be such that $G^{\vecp_{(0)}}(\tau) = G^{\vecp_{(0)}}(T_i)$.  Then,
$\payoff^{\vecp_{(0)}}_{\alg}(T_i+1, \tau) \leq (T_i - t)/2$.
\end{lemma}
\begin{proof}
We will instead show that $\E[R^{\vecp_{(0)}}_{\alg}(T_i+1, \tau)] \geq 0$ and
observe that both experts have equal payoffs in the time-steps $(T_i+1, \tau)$
since, $G^{\vecp_{(0}}_{\alg}(T_i) = G^{\vecp_{(0}}_{\alg}(\tau)$.

We shall construct a new reward sequence $\vecp'$ such that
\begin{itemize}
\item $\vecp'^t = \vecp^t_0$ for all $t \leq T_i$.
\item There exists a $\tau' > T_i$ such that
$$\vecp'^{\tau'} = \vecp^{\tau}_0 = \vecp^{T_i}_0 \quad \mbox{ and } \quad \E \rfull^{\vecp'}(T_i + 1, \tau') \leq \E R^{\vecp}_{\alg}(T_i + 1, \tau).$$
\end{itemize}

In other words, we first construct a new sequence. Then we argue that the local
regret by using $\proc{Full}$ over the new sequence is better than the original
regret. Here, $\proc{Full}$ is an implementation of $\FPL$ that communicates at
every time step (essentially a non-distributed version). Finally, it is not
difficult to see that $\E \rfull^{\vecp'}(T_i + 1, \tau') \geq 0$ because
$G^{\vecp'}(T_i + 1) =G^{\vecp'}(\tau')$, which would complete the proof of the
Lemma.

Let $T_{\ell}$ be the largest communicated time step that is no larger than
$\tau$. We use the algorithmic procedure described in Figure~\ref{alg:transform}
to construct the new sequence. Notice that our construction gives the function
$G^{\vecp'}$, which indirectly gives $\vecp'$.

\begin{figure}
\begin{center}
\fbox{
\begin{minipage}{0.60 \textwidth}
\begin{tabbing}
aaa\=aaa\=aaa\=aaa\=aaa\= \kill
{\bf original sequence}: $\vecp_0$ \\
{\bf new sequence}: $\vecp'$. \\ \\
{\bf set} $\vecp'^t = \vecp^t_0$ for all $t \leq T_i$ \\
{\bf set} $t = T_i + 1$ \\
{\bf for} $j = 1 \ldots, \ell - 1$ \+ \\
{\bf for} $\rho = G^{\vecp_0}(T_j + 1) \ldots G^{\vecp_0}(T_{j + 1})$ \+ \\
{\bf set} $G^{\vecp'}(v) = \rho$, $t = t + 1$ \\
{\bf set} t(j + 1) = t \- \\
{\bf} \- \\
{\bf for} $\rho = G^{\vecp_0}(T_{\ell} + 1) \ldots G^{\vecp_0}(\tau)$ \+ \\
{\bf set} $G^{\vecp'}(v) = \rho$, $t = t + 1$,  \\
{\bf} \- \\
{\bf set} $\tau' = t$.
\end{tabbing}

\end{minipage}
}
\end{center}\caption{\label{alg:transform} Algorithm to construct a sequence in
Lemma~\ref{lem:three}}
\end{figure}

Roughly speaking, our new $\vecp'$ uses the ``shortest path'' to connect between $G(T_j)$ and $G(T_{j + 1})$ for all $T_j$ between $T_i$ and $T_{\ell}$. Then $\vecp'$ is concatenated with another ``shortest path'' from $T_{\ell}$ to $\tau$. For the purpose of our analysis, we also let $t(j)$ be the new time step in $\vecp'$ that corresponds with the old $T_j$ in $\vecp_0$. We shall prove the following two statements,
\begin{itemize}
\item For any $i \leq j \leq \ell - 1$,
\begin{equation}\label{eqn:localeqn}
\E[R^{\vecp_0}_A(T_{j} +1 , T_{j +1}) \mid \{T_i\}_{i \geq 1}] \geq \E \rfull^{\vecp'}(t(j) + 1, t(j+1)).
\end{equation}
\item Also,
\begin{equation}\label{eqn:localeqn2}
\E[R^{\vecp_0}_A(T_{\ell} +1 , \tau) \mid \{T_i\}_{i \geq 1}] \geq \E \rfull^{\vecp'}(t(\ell) + 1, \tau').
\end{equation}
\end{itemize}
One can see that these two statements are sufficient to prove our claim:
\begin{eqnarray*}
\E[R^{\vecp}_{\alg}(T_i + 1, \tau)\mid \{T_i\}_{i \geq 1}] &
\geq & \sum_{j = 1}^{\ell - 1}\E[\rfull^{\vecp'}(t(j) + 1, t(j+1))] + \E[\rfull^{\vecp'}(t(\ell) + 1, \tau')] \\
& = & \E[\rfull^{\vecp'}(T_i + 1, \tau')] \\
& \geq & 0.
\end{eqnarray*}
We now move to prove (\ref{eqn:localeqn}) and(\ref{eqn:localeqn2}). Specifically, we only demonstrate the proof of (\ref{eqn:localeqn}) and the proof for (\ref{eqn:localeqn2}) would be similar.

Without loss of generality, we may assume that $T_{j + 1} - T_j \leq 4\lambda$ for any $i \leq j \leq \ell - 1$ since if within one whole cycle there is no communication, the expected regret for this cycle is $0$.

We consider the following three cases.

\noindent{Case 1.} $T_j$ and $T_{j + 1}$ are on the same slope of a cycle (i.e. $G(t)$ is monotonic between $T_j$ and $T_{j + 1}$). In this case, $t(j + 1) - t(j) = T_{j - 1} - T_j$. With straightforward calculation, we can see that $\proc{Full}$ is always better on $\vecp'$. \\

\noindent{Case 2.} There is only one zig-turn (namely, at time $T_z$) between $T_j$ and $T_{j + 1}$. Furthermore, we may assume $|T_z - T_j| \geq |T_z - T_{j + 1}|$. The other case can be proved similarly. Let $T'_{j + 1} = T_z - |T_z - T_{j + 1}|$.
The crucial observation here is that $G^{\vecp}(T'_{j + 1}) = G^{\vecp}(T_{j + 1})$. Since there is no communication between time $T_{j + 1}+1$ and $T'_{j + 1}$, the expected regret in this region is $0$, i.e.
$$\E[R^{\vecp}_{\alg}(T'_{j + 1}, T_{j + 1})\mid \{T_{i}\}_{i \geq 1}] = 0.$$
On the other hand, since $T'_{j + 1}$ and $T_j$ are on the same slope, running a full communication algorithm is strictly better between $T_j$ and $T'_{j + 1}$ Finally, notice that the sub-interval $G^{\vecp'}(t(j) + 1)$, ...$G^{\vecp'}(t(j+1))$ is identical to $G^{\vecp}(T_j+1), ..., G^{\vecp}(T'_{j + 1})$ by construction, we have
$$\E[\rfull^{\vecp'}(t(j)+1, t(j+1))] \geq \E[R^{\vecp}_{\alg}(T_j + 1, T'_{j + 1})] = \E[R^{\vecp}_{\alg}(T_j + 1, T_{j + 1})].$$
\\
\noindent{Case 3.} There are two zig-turns (namely $T_z$ and $T_{z'}$) between $T_j$ and $T_{j + 1}$. Let $T'_j = 2T_z - T_j$ and $T'_{j + 1} = 2T_{z'} - T_{j + 1}$. Without loss of generality, let us assume that $T'_j < T'_{j + 1}$. Our observation here is that the expected regret between $T_j + 1$ and $T'_j$ for $\alg$ is $0$. Furthermore, the expected regret between $T'_{j + 1} + 1$ and $T_{j + 1}$ is also $0$. Then we can apply the arguments appeared in Case 2 again here to show that running $\proc{Full}$ for the intervals $T'_j + 1$ and $T'_{j + 1}$ is strictly better than running $\alg$. Then we can conclude that
$\E[\rfull^{\vecp'}(t(j)+1, t(j+1))] \geq \E[R^{\vecp}_{\alg}(T_j + 1, T_{j + 1})]$ for this case as well.

\end{proof}

\end{document}